\documentclass[letterpaper, 10 pt, conference]{ieeeconf}      % Use this line for a4 paper
\usepackage{newfloat}
\usepackage{caption}
\usepackage{hyperref,breakurl}
\DeclareFloatingEnvironment[fileext=cmh,placement={!ht},name=List]{myfloat}
\usepackage{ifpdf}
\usepackage{amsfonts}
\usepackage{amssymb}
\usepackage[cmex10]{amsmath}
\usepackage{multirow}
\usepackage{tikz}

\usepackage{hyperref,breakurl}

\usepackage{url}
\usetikzlibrary{decorations.markings,arrows}
\usetikzlibrary{decorations.shapes}
\tikzset{decorate sep/.style 2 args=
 {decorate,decoration={shape backgrounds,shape=circle,shape size=#1,shape sep=#2}}}
\usetikzlibrary{plotmarks}
\usetikzlibrary{shadings}
\usepackage{graphicx}
\usepackage{float}
\usepackage{amscd,amsgen,amsfonts,amsbsy}
\usepackage{mathrsfs}
\usepackage{color}
\usepackage{textcomp}
\usepackage{latexsym,graphicx}
\usepackage[british]{babel}
\usepackage[font=small,labelfont=bf]{caption}
\usepackage{pgf}
\usepackage{multirow}

\usepackage{algorithm}
\usepackage[noend]{algpseudocode}
\algnewcommand{\algorithmicgoto}{\textbf{go to}}%
\algnewcommand{\Goto}[1]{\algorithmicgoto~\ref{#1}}%
\algnewcommand{\algoand}{\textbf{and }}
\algnewcommand{\algoor}{\textbf{or }}

\IEEEoverridecommandlockouts
\usepackage{subcaption}
\usepackage{epstopdf}
\newtheorem{prop}{Proposition}

\newtheorem{remark}{Remark}
\newtheorem{lemma}{Lemma}

\title{Vector Field Guidance for Convoy Monitoring Using Elliptical Orbits }

\author{Aseem V. Borkar$^{1}$,
		Vivek S. Borkar$^{2}$	
		and Arpita Sinha$^{3}$
\thanks{$^{1}$Aseem V. Borkar and $^{3}$Arpita Sinha are with the Interdisciplinary Program in Systems and Control Engineering, IIT Bombay, India, {\tt\small aseem@sc.iitb.ac.in, asinha@sc.iitb.ac.in}}
\thanks{$^{2}$Vivek S. Borkar is with Department of Electrical Engineering, IIT Bombay, India., {\tt\small borkar@ee.iitb.ac.in}}}

\begin{document}

\maketitle

\thispagestyle{empty}
\pagestyle{empty}

\begin{abstract}
We propose a novel vector field based guidance scheme for tracking and surveillance of a convoy, moving along a possibly nonlinear trajectory on the ground, by an aerial agent. The scheme first computes a time varying ellipse  that encompasses all the targets in the convoy using a simple regression based algorithm. It then ensures convergence of the agent  to a trajectory that repeatedly traverses this moving ellipse. The scheme is analyzed using perturbation theory of nonlinear differential equations and supporting simulations are provided. Some related implementation issues are discussed and advantages of the scheme are highlighted.
\end{abstract}

\section{Introduction}
We consider the problem of effectively protecting/surveillance of a slowly moving convoy of targets with aerial agents such as UAVs or quadrotors. For the purposes of this work, we consider a single monitoring agent modelled by unicycle kinematics, and develop a guidance law whereby it latches on to a slowly moving ellipse that encircles the convoy moving along a possibly nonlinear trajectory and moves with it. A convoy is a group of targets trailing one after the other while  moving on the same path on the ground. Elliptical orbits are more economical in terms of distance travelled for monitoring such targets, as compared to circular orbits which have been more common in  earlier works. Thus the problem is twofold: to compute the moving ellipse around the convoy and to follow a trajectory that remains faithful to traversal of this elliptical orbit as the convoy moves along a possibly nonlinear but smooth trajectory. We assume that the aerial agent moves much faster than the convoy targets. We then leverage the assumed time scale separation between the motions of the target(s) and the aerial agent to exploit some facts from the perturbation theory of differential equations.
%notably the Alekseev nonlinear variation of constants formula along with traditional Lyapunov theory, in order to analyze the coupled dynamics.
%This analysis in turn suggests further possibilities which will be presented in the full journal version of this paper and are mentioned as an observation at the end.\\

We briefly recall here some related works, the reader is referred to \cite{targetsurvey} for an extended survey. There is a lot of literature on guidance strategies to follow circular orbits around stationary and moving targets. The closest in spirit to our work are the guidance laws based on appropriately designing the relevant vector fields for target tracking applications. For example, vector field based guidance laws have been used for tracking of a single target \cite{frew_circle_standoff}  or close groups of targets \cite{tsourdos_journal} with multiple UAVs while maintaining a minimum stand-off distance from the targets. In both cases the tracking UAVs achieve a phase separation on the circular orbit by controling the linear velocity with a phase error term.
\cite{frew_racetrack} extends the work in  \cite{frew_circle_standoff} to a race-track like path for tracking a convoy moving on a straight line.
\cite{frew_ellipse} transforms the guidance vector field developed in \cite{frew_circle_standoff} to track fixed elliptical orbits whose parameters depend on estimation uncertainties of the target states for a target moving in a straight line. Vector field based guidance laws for tracking circular orbits have various applications other than target tracking, e.g., atmospheric sensing \cite{vfield_cylinder}, path following  \cite{beard_journal}. Cyclic pursuit is also a popular approach where multiple agents  cooperatively converge  to a circular orbit with an equi-spaced formation, for both stationary \cite{galloway_cyc_pursuit} and moving  \cite{ma_moving_cyc_pursuit} targets. \cite{ma_moving_var_rad} extends the latter for circular orbits of varying radius around the target.

Another aspect of using circular orbits for target monitoring is formation control. \cite{zhang_formation} proposes  agent formation strategies where the agents follow circular orbits centered at target to loiter around slow moving targets and move with a fixed equi-spaced formation on the orbit when the target moves faster. \cite{leonard_formation} proposes steering  control  laws  for cooperating agents to perform and transition between two stably controlled group
motions: parallel  motion  and  circular  motion. It is also shown that this method can be used to track a point moving on a piecewise linear path.

For the problem of convoy protection one approach is to track lemniscate like orbits. \cite{oliveira_lemniscate} propose a strategy where a single UAV tracks lemniscate like orbits centered at the convoy centroid. \cite{spry_lemniscate} uses a combination of lemniscate like lateral orbits and  parameterised asymmetric longitudinal orbits to follow a convoy moving on a straight line. \cite{magnus_dubins} treats the UAVs as
Dubins vehicles and designs time-optimal paths for convoy protection of stationary ground vehicles. They propose control strategy to use these paths to monitor a convoy moving in a straight line. Unlike these strategies, our proposed strategy can easily be adapted for tracking of a convoy while maintaining a minimum stand-off distance as discussed in later sections.

The paper is organized as follows:  Section \ref{sec_algo} describes our choice of the elliptical orbit for convoy encirclement. Section \ref{sec_guidance} describes the vector field guidance strategy to guide the agent to this orbit. Section \ref{sec_sim_res} validates the guidance law and the convoy encirclement strategy through MATLAB simulations. The Appendix details some technical results used in the main text.

%We briefly recall here some of the related works, the reader is referred to the survey \cite{targetsurvey} for an extended survey. The works closest in sprit to ours, in that they are based on appropriately designing the relevant vector fields as the core idea,  are \cite{vfield_cylinder}, \cite{frew_circle_standoff}, \cite{frew_racetrack}, \cite{beard_journal}, \cite{tsourdos_journal}, which mostly deal with static or linearly moving targets and circular orbits. Other models that have been considered are cyclic pursuits \cite{galloway_cyc_pursuit}, \cite{ma_moving_cyc_pursuit}. \cite{ma_moving_var_rad}, Dubins curves \cite{magnus_dubins}, lemniscates \cite{spry_lemniscate}, \cite{oliveira_lemniscate}, etc. The works \cite{zhang_formation}, \cite{leonard_formation}, focus on arriving at suitable formations around a target. Elliptical orbits have also been considered in  \cite{frew_ellipse}.

\section{Convoy Encirclement using Elliptical Orbits}
\label{sec_algo}
 The proposed encirclement strategy discussed in this section aims to continuously encircle all the targets in the convoy as they move along some path. For this work we assume that the positions of the targets constituting the convoy are always known to the monitoring agent either through sensing or through cooperation. The encirclement strategy is implemented by  an algorithm that runs in each iteration of the guidance loop, and defines an ellipse around the  positions of the targets  at each instant of time. It is assumed that the speed of the tracking agent is $V_A\in[V_{A_{min}},\ V_{A_{max}}]$ and the speed  $V_i$ of the target $i$ is bounded above by $V_{T_{max}}$ where  $V_{T_{max}}<<V_{A_{min}}$. We follow the convention that the targets are numbered $1,...,N$ along the direction of travel for the convoy, i.e. the leading agent in the convoy is agent $N$.
We denote the set of real numbers as $\Bbb{R}$ and  the rotation matrix from the right handed global reference frame to a tilted frame with tilt angle $\theta(t)$  as
\begin{align}R_{\theta}(t)=\left[\begin{matrix}
\cos(\theta(t)) &  \sin(\theta(t))\\
  -\sin(\theta(t)) & \cos(\theta(t))
\end{matrix}\right].
\label{eqn_rot_mat}
\end{align}
  The algorithm fits a linear regression line to the target positions to define a bounding rectangle $l_1(t)\times l_2(t)$ that contains all target points either inside or on it as illustrated in Fig \ref{fig_local_frame_geometry}.
 \begin{figure}[!h]
\centering
%Requires \usepackage{graphicx}
\includegraphics[width=1\linewidth]{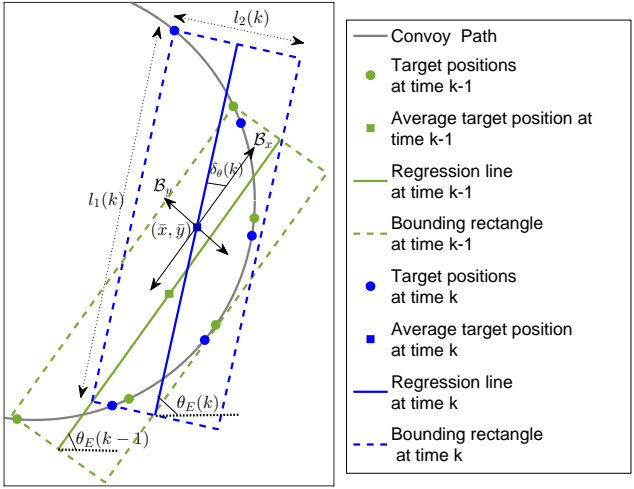}\\
\caption{Successive regression relative to the local frame $\mathcal{B}$}
\label{fig_local_frame_geometry}
\end{figure}
To encompass all the targets we consider ellipses that either contain or circumscribe this rectangle as potential paths for the monitoring agents to follow.

The  equations for computing slope $m$ and intercept $c$ of a regression line $y=mx+c$  fit to data points $(x_i,y_i)$ are:
\begin{align}
  c=\frac{\bar{y}\sum\limits_{i=1}^N x_{i}^2-\bar{x}\sum\limits_{i=1}^N x_{i}y_{i}}{\sum\limits_{i=1}^N x^2_{i}-N\bar{x}^2},\         m=\frac{\sum\limits_{i=1}^N x_{i}y_{i}-N\bar{x}\bar{y}}{\sum\limits_{i=1}^N x^2_{i}-N\bar{x}^2}
  \label{eqn_reg_params}
\end{align}
 with $\bar{x}=\frac{\sum\limits_{i=1}^N x_{i}}{N}$ and $\bar{y}=\frac{\sum\limits_{i=1}^N y_{i}}{N}$,
where $(x_i,y_i)$ represent the positions of convoy target $i$ in the global reference frame.

  \begin{figure}[!h]
\centering
%Requires \usepackage{graphicx}
\includegraphics[width=1\linewidth]{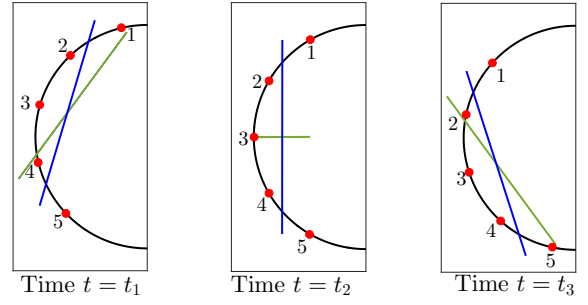}\\
\caption{The red points shown are the targets moving on the black path at three instants of time $t_1<t_2<t_3$. The green line segment is the regression line relative to global frame. The blue line segment is the result of Algorithm \ref{algo_CCR}}
\label{fig_regression_issue}
\end{figure}

\begin{remark} An issue with linear regression is that when fitting line $y=mx+c$ to the data points stacked close to the $y$ axis, linear regression yields a nearly horizontal line with large projection errors.
Thus if the regression line is fit relative to a fixed inertial frame, as when the convoy moves on a path in  the vicinity of the inertial $y$ axis, a sudden change  in inclination angle $\theta_E$ of the the regression line occurs as shown in Fig. \ref{fig_regression_issue} which is not desirable, because a segment of this line is used later to define the major axis of the encircling ellipse. We describe below a way around.
\label{rem_reg_issue}
\end{remark}

 Algorithm \ref{algo_CCR}  initialises the $\theta_E(k)$ as follows:  If the agents lie on a vertical line in the global reference frame,  then the  numerator and denominator for  regression line slope $m$ are both zero and the algorithm sets  $\theta_E(0)=\frac{\pi}{2}$. Algorithm \ref{algo_projection}   computes $l_1(0)$ as length of the line segment  joining the first and last projections of the targets on the regression line, and $l_2(0)$ as twice the maximum normal projection error $d_{max}$ from the target positions to the regression line (see Fig. \ref{fig_local_frame_geometry}). $l_2(0)>l_1(0)$ implies large projection error and regression angle is set to $\theta_E(0)=\frac{\pi}{2}-\theta_E(k)$.

To avoid the problem discussed in Remark \ref{rem_reg_issue} in the subsequent iterations of the guidance loop (indexed by iteration count $k$), Algorithm \ref{algo_CCR} selects a convoy-centric coordinate frame $\mathcal{B}(k)$ centered at the average target position $(\bar{x},\ \bar{y})$ with the $x$ axis aligned along $\theta_E(k-1)$, and a regression line $y=m'x$ (solid blue line segment in Fig. \ref{fig_local_frame_geometry}) is fit to the agent positions at the instant $k$ relative to the frame $\mathcal{B}(k)$. This regression line always passes through the origin of the coordinate frame $\mathcal{B}(k)$,  as the average value  $(\bar{x},\ \bar{y})$ of the data points always lies on the regression line by virtue of \eqref{eqn_reg_params}. The change in tilt angle $\delta_\theta(k)\in \left[\frac{-\pi}{2},\ \frac{\pi}{2}\right]$ for each iteration is computed as
$$\delta_\theta(k)=\arctan(m'(k)), \ \theta_E(k)=\theta_E(k-1) \ + \ \delta_\theta(k).$$
Algorithm \ref{algo_projection}  recomputes $\theta_E(k)$ as the inclination angle of the ray joining $(\bar{x},\bar{y})$ to the projected position $(x_{p_N},y_{p_N})$ of the target $N$ on this regression line. As a result, $\theta_E(k)\in (-\pi, \pi]$.

 Since the motion of the agents is continuous, slope $m'$ of the regression line  in the frame   $\mathcal{B}(k)$ also varies continuously with time $k$.
%, which implies that the target positions are  distributed closer to the $x$ axis of the frame $\mathcal{B}(k)$.
 As a result,  $\delta_\theta(k)$ in the local reference frame $\mathcal{B}(k)$ is  a small angle with respect to the local $x$ axis and the  issue discussed in Remark \ref{rem_reg_issue} is avoided.
%  With respect to the global frame, the regression line $y=n(k)x$ has the equation $y=m'(k)x+c(k)'$ where the global frame parameters of the regression line are :
%\begin{align}
%m'(k)=\tan(\theta_e(k)),\ c'=\bar{y}(k)-m(k)'\bar{x}(k)
%\end{align}
$l_1(k), l_2(k)$ are computed in the same manner as in the initialisation step using Algorithm \ref{algo_projection}.

\begin{algorithm}
\caption{\textproc{Convoy$\_$centric$\_$regression }}
\label{algo_CCR}
\begin{algorithmic}[1]
%\Procedure{Start}{}
{\small
\Statex {\bf Inputs: } $(x_i(k),y_i(k))\ \forall i\in\{1,...,N\}$, $k$
\Statex {\bf Functions: } \textproc{Projection}
\Statex {\bf Outputs: } $l_1(k)$, $l_2(k)$, $\theta_E(k),x_o(k),y_o(k)$
%\Statex$ (x_i(t), y_i(t))\ \forall i\in\{1,...,N\}$
\State $(\bar{x},\bar{y})=\frac{1}{N}\left(\sum\limits_{i=1}^N x_{i}(k),\sum\limits_{i=1}^N y_{i}(k)\right)$
\State $x_{arr}=\lbrace x_i(k),\forall i\in\{1,...,N\}\rbrace$
\State $y_{arr}=\lbrace y_i(k),\forall i\in\{1,...,N\}\rbrace$
\If {$k=0$}
\State $m_{n}=\sum\limits_{i=1}^N x_{i}(k)y_{i}(k)-N\bar{x}\bar{y},\ m_d=\sum\limits_{i=1}^N x^2_{i}(k)-N\bar{x}^2$
\If {$m_n=0$ and $m_d=0$}
\State $\theta_E(k)=\frac{\pi}{2}$
\Else  $\ m=\frac{m_n}{m_d}$,
\State $\theta_E(k)=\arctan(m)$
\EndIf
\State [$l_1(k),l_2(k),x_o(k),y_o(k),\theta_E(k)$]= \Statex\hspace{2cm}\textproc{Projection}$\left(x_{arr},y_{arr}, \bar{x},\bar{y},\theta_E(k) \right)$
\If {$l_1<l_2$}
        $\theta_E(k)=\frac{\pi}{2}-\theta_E(k)$
\EndIf
\Else
\For {$i\in\{1,...,N\}$}

\State $\left[\begin{matrix}x^{\mathcal{B}}_i(k) \\ y^{\mathcal{B}}_i(k)  \end{matrix}\right]=R_{\theta_E}(k-1)\left[\begin{matrix}x_i(k)-\bar{x} \\ y_i(k)-\bar{y}  \end{matrix}\right]$
 \EndFor
\State $m'_{n}=\sum\limits_{i=1}^N x^{\mathcal{B}}_{i}(k)y^{\mathcal{B}}_{i}(k)$, $\ m'_d=\sum\limits_{i=1}^N {x^{\mathcal{B}}_{i}}^2(k)$, $m'=\frac{m'_n}{m'_d}$
\State $\delta_{\theta}(k)=\arctan(m')$
%\If {k=0}$\ \theta_E(0)=\theta_E(0)+\delta_{\theta}(k)$
%\Else
\State $\theta_E(k)=\theta_E(k-1)+\delta_{\theta}(k)$
%\EndIf
%\State $m=\tan(\theta_E(k))$
\State [$l_1(k),l_2(k),x_o(k),y_o(k),\theta_E(k)$]=
\Statex\hspace{2cm} \textproc{Projection}$\left(x_{arr},y_{arr}, \bar{x},\bar{y},\theta_E(k) \right)$
\EndIf}
\end{algorithmic}
\end{algorithm}

\begin{algorithm}
\caption{\textproc{Projection}}
\label{algo_projection}
\begin{algorithmic}[1]
%\Procedure{Start}{}
\small
\Statex {\bf Inputs:} $x_{arr},y_{arr},\bar{x},\bar{y},\theta_E(k)$
\Statex {\bf Outputs:} $l_1(k),l_2(k),x_o(k),y_o(k),\theta_E(k)$
\State $x_{min}=0,x_{max}=0$, $d_{max}=0$
\For {$i\in \{1,...,\vert x_{arr}\vert \}$}
\State $\left[\begin{matrix}x_r\\ y_r \end{matrix}\right]=R_{\theta_E}(k)\left[\begin{matrix} x_{arr}[i]-\bar{x} \\ y_{arr}[i]-\bar{y} \end{matrix}\right]$
\If{$d_{max}\leq \vert y_r \vert $}
\State $d_{max}=\vert y_r \vert$
\EndIf
\If{$x_{min}\geq x_r$}
\State $x_{min}=x_r$
\EndIf
\If{$x_{max}\leq x_r$}
\State $x_{max}=x_r$
\EndIf
\If{$i=N$}
\State $x^\mathcal{B}_N=x_r$
\EndIf
\EndFor
%\State $d_{max}=\max\limits_{i\in\{1,...,N\}}\vert y_r[i] \vert$
\State$\left[\begin{matrix}x_{p_{min}}\\ y_{p_{min}} \end{matrix}\right]=R_{\theta_E}^{-1}(k)\left[\begin{matrix} x_{min} \\ 0 \end{matrix}\right]+\left[\begin{matrix} \bar{x} \\ \bar{y} \end{matrix}\right]$
\State$\left[\begin{matrix}x_{p_{max}}\\ y_{p_{max}} \end{matrix}\right]=R_{\theta_E}^{-1}(k)\left[\begin{matrix} x_{max} \\ 0 \end{matrix}\right]+\left[\begin{matrix} \bar{x} \\ \bar{y} \end{matrix}\right]$
\State$\left[\begin{matrix}x_{p_N}\\ y_{p_N} \end{matrix}\right]=R_{\theta_E}^{-1}(k)\left[\begin{matrix} x^\mathcal{B}_N \\ 0 \end{matrix}\right]+\left[\begin{matrix} \bar{x} \\ \bar{y} \end{matrix}\right]$

\State $x_o(k)=\frac{x_{p_{min}}+x_{p_{max}}}{2}$ \State $y_o(k)=\frac{y_{p_{min}}+y_{p_{max}}}{2}$
 \State $l_1(k)=\sqrt{(x_{p_{min}}-x_{p_{max}})^2+(y_{p_{min}}-y_{p_{max}})^2}$
\State $l_2(k)=2d_{max}$
\State $\theta_E(k)=\arctan2(y_{p_N}-\bar{y},x_{p_N}-\bar{x})$

\end{algorithmic}

\end{algorithm}

From  Algorithm \ref{algo_CCR} we have the tilt angle
$\theta_E(k)$ of the major axis relative to the global reference frame.
The lengths of the semi-major and minor axes can be selected according to the objective of the monitoring mission.  Henceforth we consider the minimum area ellipse circumscribing the bounding rectangle around the convoy. By Lemma \ref{lem_min_A} of the Appendix, the minimum area ellipse $\mathcal{E}: \frac{x^2}{a^2}+\frac{y^2}{b^2}=1$ with $a>b>0$ that circumscribes the bounding rectangle around the convoy has semi-major axes $a=\frac
{l_1}{\sqrt{2}}$ and semi-minor axes $b=\frac
{l_2}{\sqrt{2}}$ (shown as the blue ellipse  in Fig. \ref{fig_ellipse_geometry}).
 Some monitoring applications such as \cite{frew_circle_standoff} may require that the monitoring agent must maintain a minimum stand-off distance $d_s$ from the targets. To ensure this, instead of the $l_1\times\l_2$ rectangle, the $l_{1s}\times\l_{2s}$ rectangle centered at $(x_o,y_o)$ is considered with $l_{1s}=l_1+2d_s$ and $l_{2s}=l_2+2d_s$, whereby the same algorithm  guarantees a minimum stand off distance of the elliptical path from the convoy agents (shown as the black ellipse in Fig. \ref{fig_ellipse_geometry}).
 \begin{figure}[!h]
\centering
%Requires \usepackage{graphicx}
\includegraphics[width=1\linewidth]{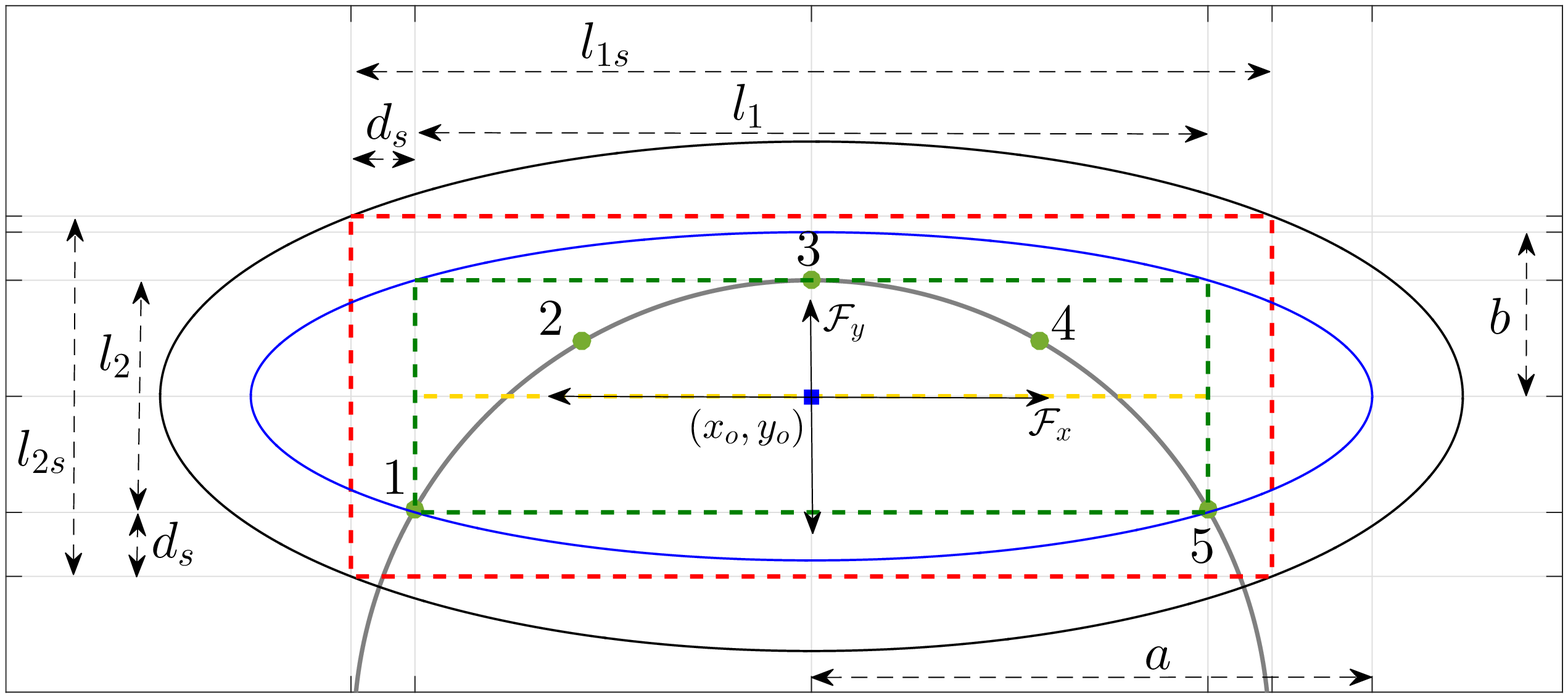}\\
\caption{Ellipse centered frame $\mathcal{F}$ with origin $(x_o,y_o)$. The yellow dashed line is the regression segment $l_1$. The blue ellipse is the minimum area ellipse circumscribing the bounding rectangle $l_1\times l_2$ (shown with green dashed line). The black ellipse circumscribing the red dashed rectangle guarantees a minimum stand off distance $d_s$ from all the targets in the convoy.}
\label{fig_ellipse_geometry}
\end{figure}

Most agents in practice have an upper bound $\omega_{max}$ on their angular speed. Assume that they are described by the unicycle model with state space representation
\begin{align}
\dot{x}_{A}(t)=V_A\cos(\psi_{A}(t)) &,\
\dot{y}_{A}(t)=V_A\sin(\psi_{A}(t)),\nonumber\\
\dot{\psi}_{A}(t)&=\omega_{A}(t)
\label{eqn_unicycle}
\end{align}
where $(x_{A}(t),y_{A}(t))$ are the agent position coordiantes, $\psi_{A}(t)$ the heading angle, $V_A\in[V_{A_{min}},\ V_{A_{max}}]$ is the commanded constant linear speed, and $\omega_{A}(t)$  the commanded angular velocity of  the monitoring agent satisfying $\vert\omega_{A}(t)\vert \leq \omega_{max} $. By Lemma \ref{lem_min_R} of Appendix, minimum radius of curvature of the ellipse $\mathcal{E}$ is $\mathcal{R}_{min}=\frac{b^2}{a}$. To ensure that   $\omega_{A}<\omega_{max}$ while following the elliptical path, $\mathcal{R}_{min}$ and the minimum turn radius of the agent $\mathcal{R}_{A}=\frac{V_{{A}_{max}}}{\omega_{max}}$ must satisfy $\mathcal{R}_{min}\geq \mathcal{R}_A$, which implies $a\geq\frac{V_{{A}_{max}}}{\omega_{max}}$ and $\ b\geq\sqrt{\frac{aV_{{A}_{max}}}{\omega_{max}}}$.
%=\sqrt{\frac{l_1V_a}{\sqrt{2}\omega_{max}}}$
%\label{eqn_a_b_lb}
%\end{align}
Assuming that $\delta_{\theta_E}$ is small, the maximum relative velocity between the target and the agent is $V_{R_{max}}=V_{A_{max}}+V_{T_{max}}$. Thus we select
\begin{align}
a(k)&=\max\left\lbrace\frac{l_1(k)}{\sqrt{2}},\frac{V_{R_{max}}}{\omega_{max}}\right\rbrace , \nonumber\\ b(k)&=\max\left\lbrace \frac{l_2(k)}{\sqrt{2}},\ \sqrt{\frac{a(k)V_{R_{max}}}{\omega_{max}}}\right\rbrace ,
\label{eqn_a_b_choice}
\end{align} which ensures that the circumscribing elliptical orbit of minimum area is selected as long as it doesn't violate the minimum turn radius or minimum speed of the monitoring agent on the ellipse.

%For this document we consider the following two cases:\\
%\noindent {\bf Application 1}: For applications where the monitoring agent has some finite sensing capablities and the objective is to keep the targets in close view, the minimum area ellipse circumscribing the bounding rectangle around the convoy is considered. \\
%\noindent {\bf Application 2}:  If the area around the convoy is to be monitored, or a minimum stand-off distance from the targets in the convoy is to be maintained, then the  minimum area ellipse circumscribing a concentric rectangular box larger than the bounding box by an appropriate ammount is considered

\section{Guidance Strategy}
\label{sec_guidance}
Since the proposed algorithm assigns an ellipse around the convoy of interest, a guidance strategy is essential for tracking this elliptical path around the convoy. This strategy must be able to guide the monitoring agent from any initial pose to any ellipse of interest defined in 2D space. To simplify analysis it is assumed that the speed of the monitoring tracking agent is  greater than the convoy's maximum speed. Also, the monitoring agent is characterized by a unicycle kinematic model described in \eqref{eqn_unicycle}. For the case of an ellipse having its axes aligned along the 2D reference frame and centered at the origin, the equation of the ellipse is given by $\frac{x^2}{a^2}+\frac{y^2}{b^2}=1$. Differentiating this equation with respect to the $x$ coordinate, the tangential direction at a point $(x,y)$ on the ellipse in the counter-clockwise direction is given by
\begin{align}
\psi_T= \arctan2\left(dy,dx\right)=\arctan2\left(b^2x,-a^2y\right)
\label{eqn_chi_T}
\end{align}
where the term $\arctan2$ yields $\psi_T\in(-\pi,\ \pi]$. This function is undefined  at $(0,0)$ where we set it equal to zero.

 Consider the family of concentric ellipses $\frac{x^2}{a^2}+\frac{y^2}{b^2}=c$ with $c>0$. Any point $(x_p,y_p)\in \Bbb{R}^2$ lies on a unique ellipse $\frac{x^2}{a^2}+\frac{y^2}{b^2}=c_p$ from this family, where $c_p=\frac{x_p^2}{a^2}+\frac{y_p^2}{b^2}$ and $\psi_T\vert_{(x_p,y_p)}$ gives the tangential heading angle along this ellipse as shown in  Fig. \ref{fig_ellipse}. All points in $\Bbb{R}^2$ lying on the a line $y=mx$ for any slope $m$ result in the same value of   $\psi_T$ .

 Suppose the desired ellipse to be followed is  $\frac{x^2}{a^2}+\frac{y^2}{b^2}=1$ with $a>b$ and the agent position coordinates are $(x_{A}(t),\ y_{A}(t))$. Define $\gamma(t)=\frac{x_{A}^2(t)}{a^2}+\frac{y_{A}^2(t)}{b^2}$. The desired heading $\psi_D(t)$ for the monitoring agent is, for $k_{\gamma} > 0$,
 \begin{align}
 \psi_{D}(t) &= \psi_{T}(t) + \psi_{O}(t),
\label{eqn_chi_d} \\
\psi_{T}(t)=\psi_T\vert_{(x_{A}(t),y_{A}(t))}, & \
 \psi_{O}  =\arctan\left(k_\gamma(\gamma(t) -1)\right) . \nonumber
 \end{align}
 \begin{figure}[!h]
\centering
%Requires \usepackage{graphicx}
\includegraphics[width=0.9\linewidth]{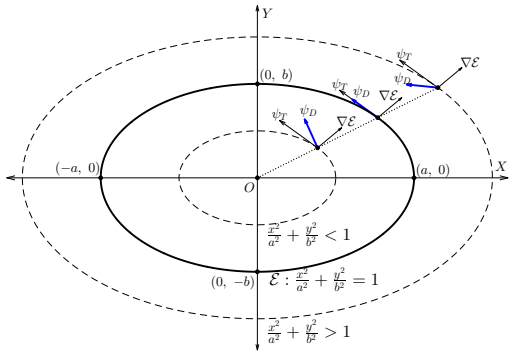}\\
\caption{The chosen agent heading $\psi_d$ shown inside, outside, and on the desired  ellipse $\mathcal{E}:\ \frac{x^2}{a^2}+\frac{x^2}{b^2}=1$}
\label{fig_ellipse}
\end{figure}

\begin{prop}
If $\psi_{A}(t)=\psi_{D}(t) $ in \eqref{eqn_unicycle}, then starting at any initial position $(x_{A}(0), y_{A}(0))$, the  agent  asymptotically converges to the desired ellipse $\mathcal E:\ \frac{x^2}{a^2}+\frac{y^2}{b^2}=1$.
\label{prop_lim_cycle}
\end{prop}

 \begin{proof}
 For $\psi_{A}(t)=\psi_{D}(t) $ the idealized unicycle agent has the following state equations:
 \begin{align}
 \dot{x}_{A}(t)=V_{A}\cos(\psi_{D}(t)),\ \dot{y}_{A}(t)=V_{A}\sin(\psi_{D}(t))
\label{eqn_unicycle_chi_d}
 \end{align}
Define the Lyapunov candidate function
\begin{align}
\mathcal{V}=\left(\frac{x_{A}^2(t)}{a^2}+\frac{y_{A}^2(t)}{b^2}-1\right)^2
\label{eqn_lyap_V}
\end{align}
By differentiating with respect to time we get
\begin{align}
\frac{d\mathcal{V}}{dt}
=2\left(\frac{x_{A}^2(t)}{a^2}+\frac{y_{A}^2(t)}{b^2}-1\right) \left\langle \left[\begin{matrix}
 2x_{A}(t)/a^2 \\ 2y_{A}(t)/b^2
\end{matrix} \right],\ \left[\begin{matrix}
\dot{x}_{A}(t) \\ \dot{y}_{A}(t)
\end{matrix} \right] \right\rangle \nonumber \\
=2V_{A}\left(\frac{x_{A}^2(t)}{a^2}+\frac{y_{A}^2(t)}{b^2}-1\right) \left\langle \nabla \mathcal{E}_{A},\ \left[\begin{matrix}
\cos(\psi_D(t)) \\ \sin(\psi_D(t))
\end{matrix} \right] \right\rangle
\label{eqn_lyap_V_dot}
\end{align}
where $\nabla \mathcal{E}_{A}$ is the gradient of the ellipse of the family $\frac{x^2}{a^2}+\frac{y^2}{b^2}=c$ at $(x_{A}(t),y_{A}(t))$ in the  outward normal direction. \\
\noindent {\it Case 1:} If the agent position $(x_{A}(t),\ y_{A}(t))$ is inside the desired ellipse, $\frac{x_{A}^2(t)}{a^2}+\frac{y_{A}(t)^2}{b^2}<1$, hence  $\gamma(t)\in[0\ 1)$. For $(x_{A}(t),\ y_{A}(t))=(0,0)$, $\frac{d\mathcal{V}}{dt}=0$, but  $(0,0)$ is not an equilibrium point of \eqref{eqn_unicycle_chi_d} by design (note that the vector field is discontinuous at $(0,0)$)  and the state trajectory moves out of $(0,0)$. For $(x_{A}(t),\ y_{A}(t))\neq(0,0)$ inside $\mathcal{E}$,  $\psi_{T}(t)$ is the counter-clockwise tangential direction along the ellipse perpendicular to $\nabla \mathcal{E}_{A}$ and $\psi_{O}(t) \in \left(\frac{-\pi}{2}, 0\right)$, the agent velocity vector with $\psi_A(t)=\psi_D(t)=\psi_{T}(t)+\psi_{O}(t)$ makes an acute angle with $\nabla \mathcal{E}$ as shown in the Fig. \ref{fig_ellipse}. Thus the inner product term in \eqref{eqn_lyap_V_dot} is positive. Therefore, as $\frac{x_{A}^2(t)}{a^2}+\frac{y_{A}^2(t)}{b^2}-1<0$, $\frac{d\mathcal{V}}{dt}< 0$ for all points inside $\mathcal{E}$ except $(0,0)$.

\noindent {\it Case 2:} If $(x_{A}(t),\ y_{A}(t))$ is outside the desired ellipse,$\frac{x_{A}^2(t)}{a^2}+\frac{y_{A}^2(t)}{b^2}>1$, hence $\gamma(t)\in(1,\ \infty)$ and $\psi_{O}(t) \in \left(0, \frac{\pi}{2}\right)$. Thus the agent velocity vector with heading direction $\psi_{A}(t)=\psi_D(t)=\psi_{T}(t)+\psi_{O}(t)$ makes an obtuse angle with $\nabla \mathcal{E}_{A}$ as shown in the Fig. \ref{fig_ellipse} and the inner product term in \eqref{eqn_lyap_V_dot} is negative. Therefore as $\frac{x_{A}^2(t)}{a^2}+\frac{y_{A}^2(t)}{b^2}-1>0$ at any point outside $\mathcal{E}$, $\frac{d\mathcal{V}}{dt}< 0$.

\noindent {\it Case 3:}  If $(x_{A}(t),\ y_{A}(t))$ is on the desired ellipse $\frac{x_{A}^2(t)}{a^2}+\frac{y_{A}^2(t)}{b^2}=1$, then  $\frac{d\mathcal{V}}{dt}=0$. Since $\gamma(t)=1$, $\psi_{O}(t)=0$ and $\psi_{A}(t)=\psi_{T}(t)$, which is the tangential direction along the ellipse. Thus the state trajectory always remains on  $\mathcal{E}$, implying that it is a positively invatiant set.

Thus $\frac{d\mathcal{V}}{dt}\leq 0$ for all $(x_{A}(t),y_{A}(t))\in\Bbb{R}^2$ and $\frac{d\mathcal{V}}{dt}= 0$  for $(x_{A}(t),y_{A}(t))\in E$, where $E=\mathcal{E}\cup \{(0,0)\}$. Since $\mathcal{E}$ is the largest invariant subset in $E$,  from any initial  $(x_{A}(0),y_{A}(0))$, $(x_{A}(t),y_{A}(t))$ approaches $\mathcal{E}$ asymptotically by LaSalle's invariance principle \cite{khalil}.
%The vector field resulting from for the system \eqref{eqn_unicycle_chi_d} is shown in Fig. \ref{fig_vector_field}
 \end{proof}
% \begin{figure}[!h]
%\centering
%%Requires \usepackage{graphicx}
%\includegraphics[width=0.6\linewidth]{fig_vector_field.eps}\\
%\caption{Vector field for \eqref{eqn_unicycle_chi_d}, for   the desired  ellipse is $\mathcal{E}:\frac{x^2}{5^2}+\frac{x^2}{3^2}=1$ for $k_\gamma=0.5$}
%\label{fig_vector_field}
%\end{figure}

For a moving ellipse we now analyse the idealized coupled agent-ellipse dynamics using perturbation theory for differential equations. (Note that the actual algorithms use discretized versions of these.) Let $z_A(t) = [x_A(t), y_A(t)]^T$, $z_T(t) = [x_o(t), y_o(t), a(t), b(t), \theta_E(t)]^T$ denote respectively the position of the agent, and the vector of the ellipse parameters (center position, axis lengths and tilt) at time $t$ given by Algorithm \ref{algo_CCR} and \eqref{eqn_a_b_choice}.  Let their  respective dynamic laws be given by
\begin{align}
\dot{z}_A(t) &= h(z_A(t), z_T(t)), \label{agentdynamics} \\
\dot{z}_T(t) &= \epsilon g(z_T(t)). \label{targetdynamics}
\end{align}
Here $\epsilon > 0$ is small, so the target moves on a slower time scale than the agent. Let $\mathcal{D} := \Bbb{R}^4\times(-\pi,\pi]$. We assume $h : \Bbb{R}^2\times\mathcal{D} \mapsto \Bbb{R}^2$ and $g : \mathcal{D}  \mapsto \Bbb{R}^2$ are Lipschitz, so (\ref{agentdynamics}), (\ref{targetdynamics}) are well posed. Also consider
\begin{equation}
\dot{\tilde{z}}_A(t) = h(\tilde{z}_A(t), z_T^*), \label{frozen}
\end{equation}
i.e., dynamics for the agent when target  is stationary at point $z_T^*$. From  Proposition \ref{prop_lim_cycle}, we know that this has a limit cycle (i.e., a periodic solution) $\chi(t, z^*_T), t \geq 0,$ parametrized by $z_T^*$. In fact, it traces an ellipse with center and orientation given by $z_T^*$. Then
\begin{equation}
\dot{\chi}(t) = h(\chi(t), z_T^*). \label{cycle}
\end{equation}
Let $w(t) := \chi(t, z_T(t)).$ This is the equation for the agent trajectory where we have made the parameter $z^*_T$ of the above  periodic solution time-varying, albeit on a slower time scale as per (\ref{targetdynamics}).
Letting $\nabla^z :=$ the gradient w.r.t.\ $z$,
\begin{align}
\dot{w}(t) &= \frac{\partial \chi}{\partial t}(t, z_T(t)) + \langle \nabla^z\chi(t, z_T(t)), \dot{z}_T(t) \rangle \nonumber \\
&= h(\chi(t, z_T(t)), z_T(t)) + \epsilon\langle \nabla^z\chi(t, z_T(t)), g(z_T(t)) \rangle \nonumber \\
&= h(w(t), z_T(t)) + \epsilon\eta(t), \label{perturb}
\end{align}
where $\eta(t) := \langle \nabla^z\chi(t, z_T(t)), g(z_T(t)) \rangle.$
This can be viewed as a perturbation of (\ref{agentdynamics}). Next we use the Alekseev formula \cite{Alekseev} to give an explicit expression for the error between the two. Let $Z(t) := [z_A(t), z_T(t)]^T, \ F_{\epsilon} (z, z') = [h(z, z')^T, \epsilon g(z')^T]^T$.  Then  the combined dynamics (\ref{agentdynamics})-(\ref{targetdynamics}) is
$\dot{Z}(t) = F_{\epsilon}(Z(t)).$
Let $Z(t, \tau; \hat{z}), t \geq \tau,$ denote its solution for $Z(\tau) = \hat{z}$. Consider the linearization of (\ref{agentdynamics})-(\ref{targetdynamics}) given by
\begin{equation}
\delta\dot{Z}(t) = DF_{\epsilon}(Z(t, \tau; \hat{z}))\delta Z(t), \ t \geq \tau, \label{linearize}
\end{equation}
where $DF_{\epsilon}(\cdot)$ is the Jacobian matrix of $F_{\epsilon}$. This is a time-varying linear system. Let $\Phi_{\epsilon}(t, \tau; \hat{z}), t \geq \tau,$ denote its fundamental matrix satisfying
$$\dot{\Phi}_{\epsilon}(t, \tau; \hat{z}) = DF_{\epsilon}(Z(t, \tau; \hat{z}))\Phi_{\epsilon}(t, \tau; \hat{z}), \ t \geq\tau,$$
with $ \Phi_{\epsilon}(\tau, \tau; \hat{z}) = I :=$ the identity matrix,  and
$\delta Z(t) = \Phi_{\epsilon}(t, \tau; \hat{z})\delta Z(\tau).$
Then by Alekseev's  formula (\cite{Alekseev}, see also Lemma 3 of \cite{Brauer}), we have
\begin{align}
z_A(t) = w(t) - \epsilon\int\limits_0^t\bar{\Phi}_{\epsilon}(t, \tau; w(\tau))\eta(\tau)d\tau, \ t \geq 0. \label{error2}
\end{align}
where $\bar{\Phi}_{\epsilon}(\cdot)$ is the submatrix of $\Phi_{\epsilon}(\cdot)$ formed by its top half rows.
This gives an explicit expression for the error between the actual agent dynamics $z_A(\cdot)$ and its ideal dynamics $w(\cdot)$. That this error remains small for all time  can be proved by a variant of Theorem 1, p.\ 339, \cite{Hirsch}. In \cite{Hirsch}, this result is stated for stable equilibria, but the same proof works here in view of the explicit Lyapunov function for (\ref{frozen}) exhibited earlier, as we argue below.

For an arbitrary ellipse in $\Bbb{R}^2$, if $\mathcal{F}$ is  the ellipse centric frame with  origin at ellipse center $(x_o(t),y_o(t))$ and tilt angle $\theta_E(t) \in(-\pi,\ \pi]$ relative to the global reference frame, then in the frame $\mathcal{F}$  the agent position and heading are given by
\begin{align}
\left[\begin{matrix}
x_E(t) \\ y_E(t)\\ \psi_E(t)
%\\ \dot{x}_e \\ \dot{y}_e
\end{matrix}\right]=\left[\begin{matrix}
 R_{\theta_E}(t) \left[ \begin{matrix}x_A(t)-x_o(t) \\ y_A(t)-y_o(t) \end{matrix}\right]\\
\psi_A(t)-\theta_E(t)
\\
%\left[ \begin{matrix}
%\cos(\theta_E) & \sin(\theta_E ) \\ -\sin(\theta_E) & \cos(\theta_E)
%\end{matrix}  \right] \left[ \begin{matrix}\dot{x} \\ \dot{y}\end{matrix}\right]
\end{matrix}\right].
\end{align}

Consider the dynamics in the reference frame $\mathcal{F}$ centered at $(x_o,y_o)$ for the moving ellipse (shown in Fig. \ref{fig_ellipse_geometry}).
%\begin{eqnarray*}
%\frac{x_E(t)^2}{a^2(t)}+ \frac{y_E(t)^2}{b^2(t)} = 1.
% \end{eqnarray*}
  We assume that $\dot{x}_o(t), \dot{y}_o(t), \dot{\theta}_E(t), \dot{a}(t), \dot{b}(t)= O(\epsilon)$ uniformly in $t$. Consider the time-dependent Lyapunov function
 \begin{align}
\tilde{\mathcal{V}}(z_A, t) :=\left(\frac{x_E(t)^2}{a^2(t)}+ \frac{y_E(t)^2}{b^2(t)} -1\right)^2.
\end{align}
Let $\zeta(t)=\frac{x_E(t)^2}{a^2(t)}+ \frac{y_E(t)^2}{b^2(t)} -1$. Then by arguments similar to those of Proposition \ref{prop_lim_cycle}, we have
\begin{align}
\frac{d}{dt}\tilde{\mathcal{V}}(z_A(t),t) = 2V_A\zeta(t) \left\langle \nabla \mathcal{E}_{A},\ \left[\begin{matrix}
\cos(\psi_d) \\ \sin(\psi_d)
\end{matrix} \right] \right\rangle \nonumber \\
+ \  2V_A\zeta(t)\left\langle\nabla\mathcal{E}_A, \left[\begin{matrix}\delta_1(t)\\ \delta_2(t))\end{matrix}\right] \right\rangle,
\end{align}
where $\|\delta_i(t)\| = O(\epsilon), \ i = 1,2$. Label the two summands on the right as $\xi_1(t), \xi_2(t)$ resp. Then $\|\xi_2(t)\| \leq K\epsilon$ for a constant $K$ that can be estimated in terms of the problem parameters. Consider a point outside the current ellipse. Then as in Proposition 1, $\zeta(t) > 0$ and $\xi_1(t) < 0$. Then as long as
$\|\xi_1(t)\| > K\epsilon$,
we have $\frac{d}{dt}\tilde{\mathcal{V}}(z_A(t), t) < 0$, implying not only stability (i.e., the trajectory of the agent remains bounded as long as that of the target does), but by the LaSalle invariance principle, that the trajectory converges to the set
$$\left\{(x_A, y_A) :  \Gamma(t) := \left|\left\langle\nabla\mathcal{E}_A, \left[\begin{matrix}\cos(\psi_d)\\ \sin(\psi_d)\end{matrix}\right] \right\rangle  \right| \leq K\epsilon \right\}.$$
A similar conclusion holds if the initial condition is inside the moving ellipse. Since $\Gamma(t)$ vanishes only when the agent is exactly on the desired moving ellipse, it follows that the agent converges to an $O(\epsilon)$ neighborhood thereof.

A similar analysis can be used to establish robustness to small errors, e.g., in numerical computation or due to noise.

\bigskip

\begin{remark} The above suggests the use of $\epsilon\eta(\cdot)$ as explicit additive control for the agent's dynamics (\ref{agentdynamics}) as in (\ref{perturb}), in order to achieve \textit{exactly} the desired trajectory $w(\cdot)$. This, however, would require advance knowledge of target motion.
\end{remark}
\bigskip
\begin{remark}  For related results on stability  of slowly varying linear systems that give a handle on $\|\Phi_{\epsilon}(t, \tau; \hat{z})\|$ above, see \cite{DaCunha} and \cite{Solo}.
\end{remark}
\bigskip

The proposed guidance law can handle both counter-clockwise and clockwise path following along the ellipse using the heading commands $\psi_{d_{ccw}}(t)$ and $\psi_{d_{cw}}(t)$ given by
\begin{align}
\psi_{D_{ccw}}(t)&= \psi_{T_{ccw}}(t)+\psi_{O}(t), \nonumber\\
\psi_{D_{cw}}(t)&= \psi_{T_{cw}}(t)-\psi_{O}(t),
\label{eqn_ccw_cw}
\end{align}
where  $$\psi_{T_{ccw}}(t)=\arctan2 \left(b^2(t)x_E(t),-a^2(t)y_E(t)\right),$$ $$\psi_{T_{cw}}(t)=\arctan2 \left(-b^2(t)x_E(t),a^2(t)y_E(t)\right),$$ and $\psi_O(t)=\arctan\left(k_\gamma(\gamma(t)-1)\right),$ for $\gamma(t)=\frac{x_{E}^2(t)}{a^2(t)}+\frac{y_{E}^2(t)}{b^2(t)}$ and controller gain $k_\gamma$.  In order to ensure that the agent heading $\psi_A$ follows the desired heading $\psi_D$, we use a proportional feedback control to command the agent's angular velocity $\omega_A(t)$ defined in (\ref{eqn_unicycle}):
\begin{align}
\omega_A(t) \  (= \dot{\psi}_{A(t)}) = k_\psi (\psi_D(t)-\psi_E(t)).
\end{align}
This is the classical proportional control that acts to push $\psi_{E(t)}$ towards $\psi_{D(t)}$.

\section{Simulation Results}
\label{sec_sim_res}
To validate the vector field based guidance strategy, simulations were done for different stationary ellipses having different orientations $\theta_E$ relative to the inertial frame, with the tracking agent starting from an arbitrary initial pose. Two such cases are shown in Fig. \ref{fig_stat_ellipse_track}. In the first, the agent starts outside the ellipse and follows it in the counter-clockwise direction by tracking $\psi_{D_{ccw}}$ in \eqref{eqn_ccw_cw}. In the second, the agent starts inside the ellipse and follows it in the clockwise direction by tracking $\psi_{D_{cw}}$ in \eqref{eqn_ccw_cw}. In both cases the major and minor axes are chosen such that the minimum  turn radius on ellipse $\mathcal{R}_{min}<\mathcal{R}_A$ where $\mathcal{R}_A=\frac{V_{A_{max}}}{\omega_{max}}$ is the minimum agent turn radius. From Fig. \ref{fig_stat_ellipse_data} we see that except in the initial phase where the agent is  trying to align with the desired vector field direction, the Lyapunov function $\mathcal{V}$ decreases with time in either case  and agent angular velocity satisfies $\vert\omega_{A} \vert<\omega_{max}$. This confirms that the agent converges to the desired ellipse and thereafer traces it without violating the angular velocity constraints.
 \begin{figure}%[h]
\begin{minipage}[b]{\linewidth}
\centering
%Requires \usepackage{graphicx}
\includegraphics[width=0.8\linewidth]{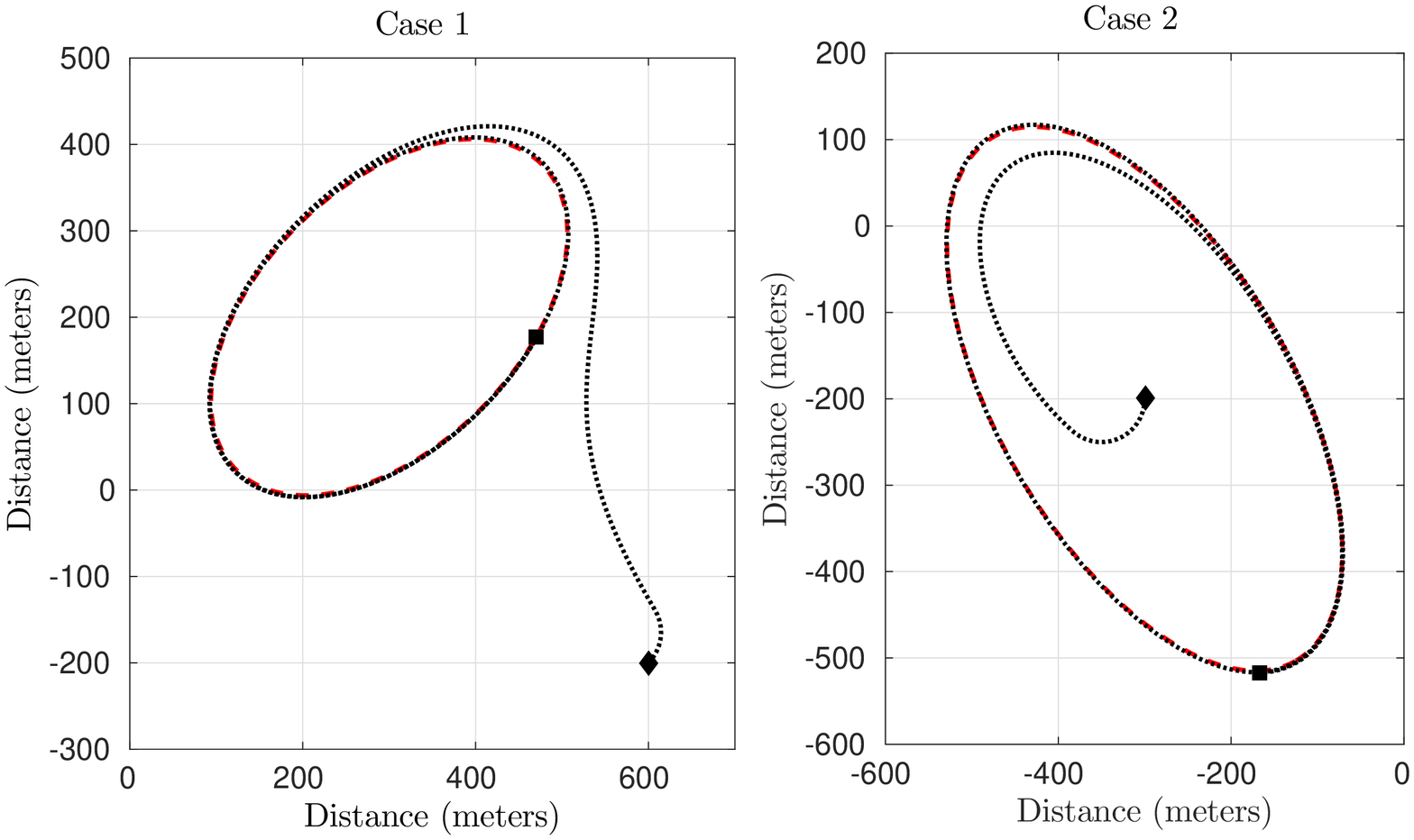}\\
\caption{ {\bf Case 1:} Ellipse parameters: $a=250$ meters, $b=150$ meters, $\theta_E=\frac{\pi}{4}$ radians, $(x_o,y_o)=(300,200)$ (in meters)
Initial agent pose: $(x_A,y_A)=(600,-200)$ (in meters), $\psi_A=\frac{\pi}{4}$ radians. {\bf Case 2:} Ellipse parameters: $a=350$ meters, $b=170$ meters, $\theta_E=-\frac{\pi}{3}$ radians, $(x_o,y_o)=(-300,-200)$ (in meters)
Initial agent pose: $(x_A,y_A)=(-300,-200)$ (in meters), $\psi_A=-\frac{\pi}{2}$ radians. Controller gains: $k_{\gamma}=0.5$, $k_\psi=1$. $V_A=15$ m/s and $\vert\omega_{max}\vert=0.3$ rad/sec.}
\label{fig_stat_ellipse_track}
\end{minipage}
\begin{minipage}[b]{\linewidth}
\centering
%Requires \usepackage{graphicx}
\includegraphics[width=1\linewidth]{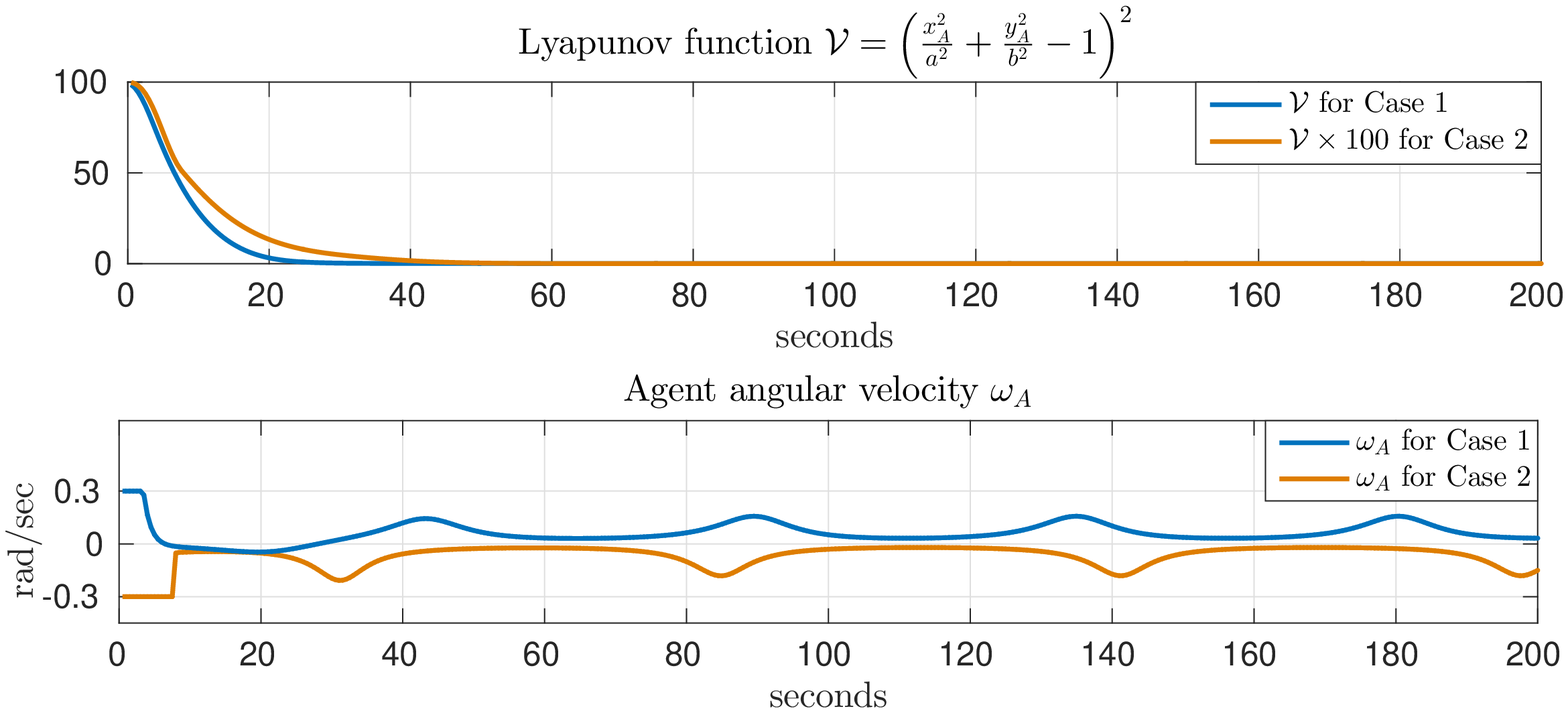}\\
\caption{The Lyapunov function and angular velocity plots for the cases in Fig. \ref{fig_stat_ellipse_track}. The Lyapunov function for {\bf Case 2} is scaled by a factor of $100$ to plot it on the same scale as {\bf Case 1}. In both cases $\vert\omega_{max}\vert=0.3$ rad/sec.}
\vspace{-0.5 cm}
\label{fig_stat_ellipse_data}
\end{minipage}
\end{figure}

To simulate Algorithm \ref{algo_CCR} and the guidance strategy for a moving convoy, the convoy is assumed to move along a Lissajous curve having a parametric equation $$x(\phi)=A\cos(\phi),\ y(\phi)=B\sin(2\phi) $$ with  $A=1500$ meters, $B=1000$ meters and $\phi\in[0,\ 2\pi)$.  The convoy comprises of five target points starting from different parameter values $\phi_i(0)$ moving at a constant parametric speed of $\dot{\phi}= 0.0012$. Thus speed of the target $i$ is
\begin{eqnarray*}
V_i &=& \sqrt{A^2\sin^2(\phi_i(t))+4B^2\cos^2(2\phi_i(t))}\dot{\phi} \\
&\leq& \sqrt{A^2+4B^2}\dot{\phi}=3 \ \mbox{m/sec}.
\end{eqnarray*}
Since the target speeds are not constant, the curve length separating the target points  varies with time as seen in Fig. \ref{fig_c_2_traj_combo} at different instances of time. The unicycle agent's linear velocity is $V_A=15$ m/sec, the permissible range of velocities is assumed to be $V_{A_{max}}=20$ meters/sec and $V_{A_{min}}=10$ m/sec and the maximum permissible angular velocity is $\vert\omega_{max}\vert=0.3$ rad/sec. The intial position and orientation of the unicycle agent is assumed to be $(x_A(0),y_A(0))=(400,400)$ (in meters) and $\psi_A(0)=\frac{-\pi}{2}$ radians. The parameters $a,b$ of the ellipse's axes are selected according to \eqref{eqn_a_b_choice}. The controller gain values are $k_\psi=1$ and $k_{\gamma}=5$.
 \begin{figure}%[h]
\begin{minipage}[b]{\linewidth}
\centering
%Requires \usepackage{graphicx}
\includegraphics[width=1\linewidth]{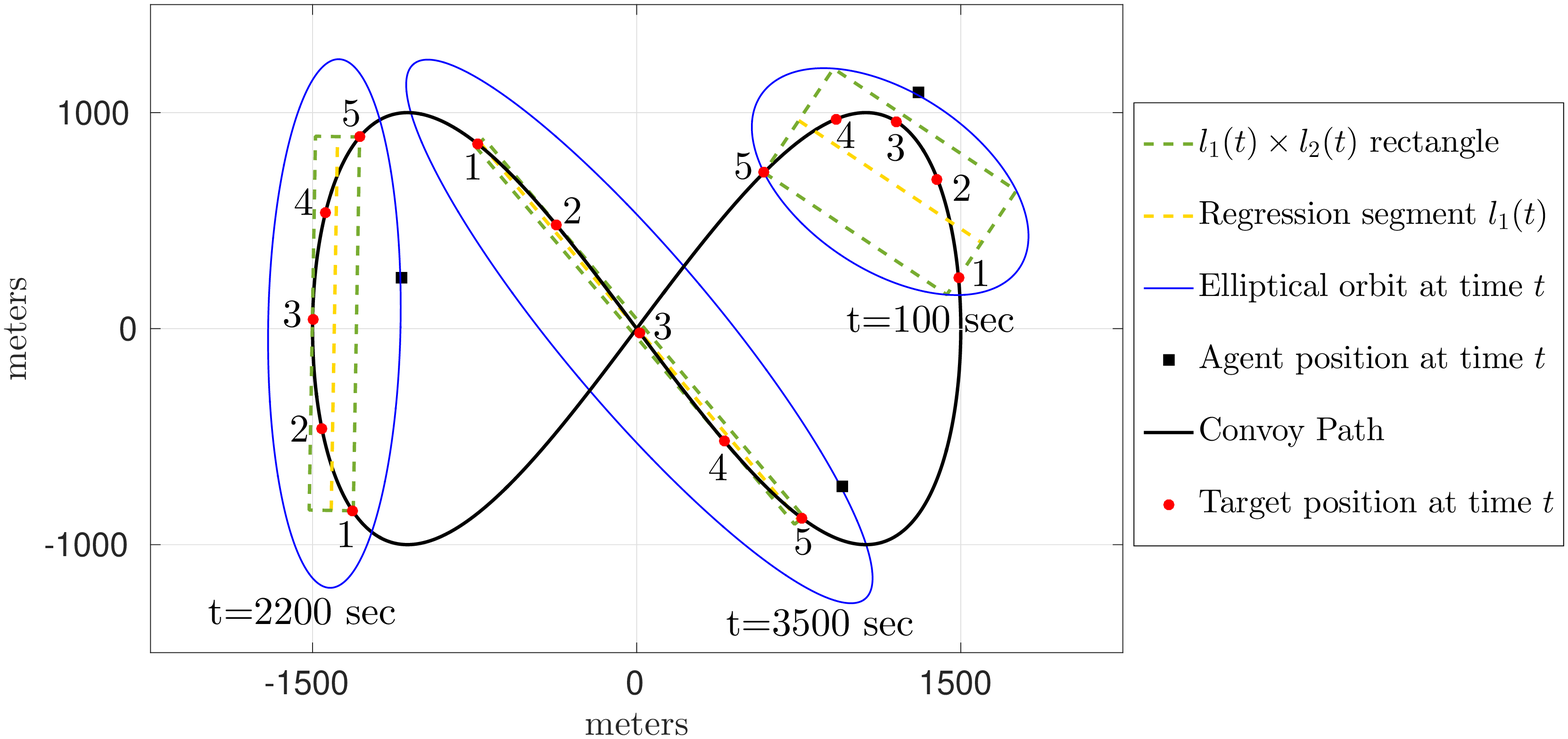}\\
\caption{Simulation 2 snapshots for time $t=100$, $2200$ and $3500$ seconds.}
\label{fig_c_2_traj_combo}
\end{minipage}
\begin{minipage}[b]{\linewidth}
\centering
%Requires \usepackage{graphicx}
\includegraphics[width=1\linewidth]{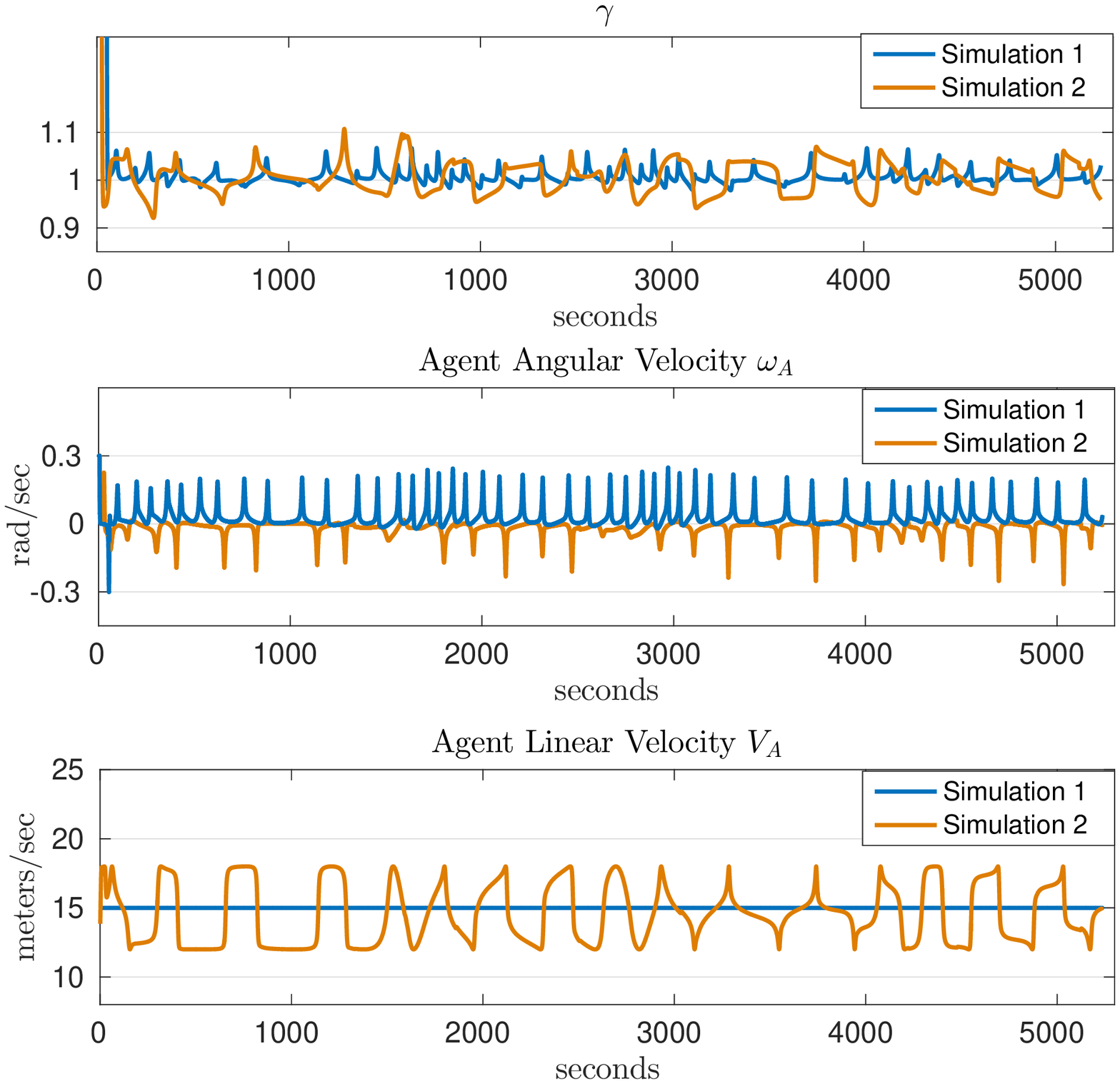}\\
\caption{$\gamma$, $\omega_A$ and $V_A$ plots for one traversal of the Lissajous curve for both simulations.}
%\vspace{-0.5 cm}
\label{fig_sim_data_combo}
\end{minipage}
\begin{minipage}[b]{\linewidth}
\centering
%Requires \usepackage{graphicx}
\includegraphics[width=1\linewidth]{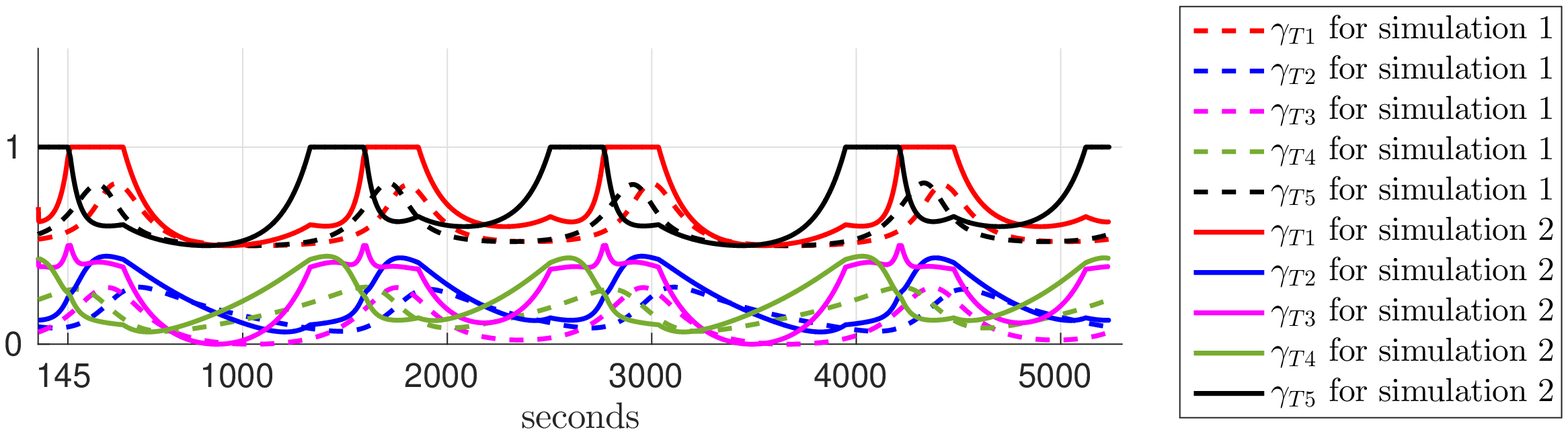}\\
\caption{$\gamma_{Ti}$ plots for one traversal of the Lissajous curve for both simulations.}
\vspace{-0.5 cm}
\label{fig_sim_convoy_data_combo}
\end{minipage}
\end{figure}

For the above agent target setting, we consider the following two simulations. For  simulation 1, the targets are spaced close to each other with $\phi_i(0)=\frac{(i-1)\pi}{20}$, and the agent follows the elliptical orbit counter-clockwise. For simulation 2, the targets are spaced farther apart with $\phi_i(0)=\frac{(i-1)\pi}{12}$,  and the agent follows the elliptical orbit clockwise. Also in simulation 2, a small constant velocity disturbance of $V_{w}=3$ m/sec at a heading of $\psi_w=\frac{\pi}{4}$ radians relative to the global reference frame is added component-wise to the $\dot{x}_A(t)$ and $\dot{y}_A(t)$ of the unicycle agent. For a UAV this is like a velocity  disturbance due to wind.

The snapshots in time of the convoy target positions and the encircling  elliptical orbits for simulation 2 are shown in Fig. \ref{fig_c_2_traj_combo}. We observe that at $t=2200$ seconds, the targets are at positions roughly along the inertial $y$ axis direction and the Algorithm \ref{algo_CCR} fits the correct regression line for these target positions avoiding the regression issue discussed in Remark \ref{rem_reg_issue}. We see that $\gamma_{Ti}(t)=\frac{x_{Ei}^2(t)}{a^2(t)}+\frac{y_{Ei}^2(t)}{b^2(t)}\leq 1$ for all targets at position $(x_{Ei}(t),y_{Ei}(t))$ in the ellipse centric frame at time $t$ for both simulations, as shown in Fig. \ref{fig_sim_convoy_data_combo}.  This validates the claim that the ellipse computed based on the outputs of Algorithm \ref{algo_CCR} (as discussed in Section \ref{sec_algo}) always circumscribes the targets.
From Fig. \ref{fig_sim_convoy_data_combo}, for the first and fifth agents in simulation 2,  $\gamma_{Ti}=1$ for certain time durations. For example, $\gamma_{T5}(t)=1$ for $t \in [0,145]$ seconds. This implies that target $5$ lies on the ellipse during this duration. This  happens when an agent is on one of the corners of the $l_1(t)\times l_2(t)$ rectangle and the ellipse parameters $a,b$ selected according to \eqref{eqn_a_b_choice} yield the minimum area ellipse circumscribing this rectangle, as shown in Fig. \ref{fig_c_2_traj_combo} at $t=100$ seconds.
  For both simulations, $\vert \omega(t)\vert<\omega_{max}$  and $\gamma(t)\approx 1$  as shown Fig. \ref{fig_sim_data_combo}. This implies that the agent follows the moving ellipse with a small error as discussed in Section \ref{sec_guidance} without violating angular velocity constraints. The peaks in the  $\omega_A(t)$ occur near the ends of the major axis where the ellipse curvature is the highest. Since the agent travels in counter-clockwise orbits in simulation 1 and clockwise orbits in simulation 2, the peaks in $\omega_A(t)$ plots for simulation 1 and simulation 2 are positive and negative respectively. As the parametric spacing between targets is less in simulation 1 than in simulation 2, the  elliptical orbits in simulation 1 are smaller and the convoy is circumnavigated more often in simulation 1. Thus the number of the peaks in $\omega_A(t)$ plots is greater in simulation 1 than in simulation 2.

\bigskip

\begin{remark} A video  of the simulations 1 and 2  for one complete traversal of the Lissajous curve by the convoy can be found at the web-link: \url{https://www.youtube.com/watch?v=57R6Tf71r5c}
\end{remark}

 \section{Conclusions}

We have proposed a novel scheme for protection and surveillance of a convoy moving on an arbitrary trajectory with minimal regularity assumptions. The scheme is based on computing a moving ellipse that circumscribes a bounding rectangle encompassing the convoy. This  rectangle is computed using a  simple regression scheme. Then by modulating the driving vector field of  the agent appropriately, it converges to a trajectory that traverses the moving elliptical orbit repeatedly. The elliptical orbits prove to be more economical than circular ones in terms of coverage. The scheme is very simple to implement and has been given a rigorous justification inclusive of error analysis using Alekseev's nonlinear variation of constants formula. The supporting simulations also show performance that matches the theoretical predictions.\\

The present analysis is restricted to a single agent. We are in the process of implementing a multi-agent version which will be reported in a sequel. Another future direction is to incorporate realistic noise models and ensure robustness vis-a-vis the same.

\section*{Acknowledgements}
The authors thank  G.\ K.\ Arunkumar  and Dr.\  Pranjal Vyas for discussions. Work done by VSB was supported in part by a J.\ C.\ Bose Fellowship.

\section*{Appendix}
\begin{lemma}
The area of an ellipse $\mathcal{E}:\frac{x^2}{a^2}+\frac{y^2}{b^2}=1$ circumscribing a rectangle of dimensions $l_1\times l_2$ with $l_1>l_2>0$ is minimized by $a=\frac{l_1}{\sqrt{2}}$, $b=\frac{l_2}{\sqrt{2}}$.
\label{lem_min_A}
\end{lemma}
\begin{proof}
For an ellipse the semi major axis $a$ and minor axis $b$ are related as: $b=a\sqrt{1-e^2}$, where $e\in[0,\ 1)$ is the eccentricity. The vertices of the rectangle at $\left(\frac{\pm l_1}{2}, \frac{\pm l_2}{2}\right)$ lie on the ellipse. Thus $\frac{l_1^2}{4a^2}+\frac{l_2^2}{4b^2}=1$, which simplifies to
$a^2= \frac{1}{4}\left(l_1^2+\frac{l_2^2}{1-e^2}\right) $
%\label{eqn_aa}
%\end{align}
As a result, the area of the ellipse is $A=\pi ab=\pi a^2\sqrt{1-e^2}=\frac{\pi}{4}\left(l_1^2\sqrt{1-e^2}+\frac{l_2^2}{\sqrt{1-e^2}}\right)$. The  first two derivatives of $A$ w.r.t.\ $e$ are: %\begin{align}
%%(e*pi*(e^2*l1^2 - l1^2 + l2^2))/(4*(1 - e^2)^(3/2))\\
\begin{eqnarray*}
\frac{dA}{de}&=&\frac{\pi}{4}\frac{e(l_2^2-l_1^2(1-e^2))}{(1-e^2)^{\frac{3}{2}}}, \ \\
\frac{d^2A}{de^2}&=&\frac{\pi}{4}\frac{l_2^2(1+2e^2) -l_1^2(1-e^2)}{(1-e^2)^{\frac{5}{2}}}.
\end{eqnarray*}%\label{eqn_d2ade2}
%\end{align}

 From first order necessary conditions for minimization,  $\frac{dA}{de}=0$, extremizer value $e^*=0,\pm\sqrt{1-\frac{l_2^2}{l_1^2}}$. Evaluating $\frac{d^2A}{de^2}$ at $e^*$, $$\left.\frac{d^2A}{de^2}\right\vert_{e^*=0}=-\frac{\pi(l_1^2-l_2^2)}{4}<0,$$ \ $$\left.\frac{d^2A}{de^2}\right\vert_{e^*=\pm\sqrt{1-\frac{l_2^2}{l_1^2}}}=\frac{\pi l_1^3}{2l_2^3}(l_1^2 - l_2^2)>0$$ 
 as $l_1>l_2>0$. As $e^*\in[0,\ 1)$, from second order sufficient conditions of minimization,  $e^*=\sqrt{1-\frac{l_2^2}{l_1^2}}$ minimizes $A$. Using the ellipse relation $\frac{b^2}{a^2}=\sqrt{1-e^2}$, for the minimum area ellipse $\frac{l_1}{a}=\frac{l_2}{b}=\eta$ for a positive constant $\eta$. Substituting this in  $\frac{l_1^2}{4a^2}+\frac{l_2^2}{4b^2}=1$, we get $\eta=\sqrt{2}$. Hence $a=\frac{l_1}{\sqrt{2}}$ and $b=\frac{l_2}{\sqrt{2}}$
\end{proof}

\begin{lemma}
For an ellipse $\mathcal{E}:\frac{x^2}{a^2}+\frac{y^2}{b^2}=1$ with $a>b>0$, the minimum radius of curvature $\mathcal{R}_{min}=\frac{b^2}{a}$
\label{lem_min_R}
\end{lemma}
 \begin{proof}
   For the parametric representation of an ellipse $\left(x(s),y(s)\right)=\left(a\cos(s),b\sin(s)\right)$, radius of curvature is
\begin{align}
\mathcal{R}=\frac{(\dot{x}^2+\dot{y}^2)^{\frac{3}{2}}}{\ddot{y}\dot{x}-\ddot{x}\dot{y}}
=\frac{\left(a^2\sin^2(s)+b^2\cos^2(s)\right)^{\frac{3}{2}}}{ab}
\end{align}
The first two derivatives of $\mathcal{R}$ with respect to $s$ are\\
%\begin{align}
$$\frac{d\mathcal{R}}{ds}=\frac{3(a^2-b^2)\sin(2s)\sqrt{a^2\sin^2(s)+b^2\cos^2(s)}}{2ab},$$
\begin{eqnarray*}
\lefteqn{\frac{d^2\mathcal{R}}{ds^2}=} \\
&&\Big(3(a^2 - b^2) \big(b^2\cos^4(s) - a^2\sin^4(s) + \\
&&2(a^2-b^2)\cos^2(s)\sin^2(s)\big)\Big)\\
&&\Big(ab\sqrt{a^2\sin^2(s) + b^2\cos^2(s)}\Big)^{-1}.
\end{eqnarray*}
%(3*(a^2 - b^2)*(b^2*cos(s)^4 - a^2*sin(s)^4 + 2*a^2*cos(s)^2*sin(s)^2 - 2*b^2*cos(s)^2*sin(s)^2))/(a*b*(a^2*sin(s)^2 + b^2*cos(s)^2)^(1/2))

%\end{align}

From first order necessary conditions for minimization,   $\frac{d\mathcal{R}}{ds}=0$,  the extremizer value  is $s^*=q\pi,\ (2q+1)\frac{\pi}{2}$ for some $q\in \Bbb{Z}^+$. Evaluating $\frac{d^2\mathcal{R}}{ds^2}$ at $s^*$ gives  $\left.\frac{d^2\mathcal{R}}{ds^2}\right\vert_{s^*=q\pi}=\frac{3(a^2 - b^2)}{a}>0$ and $\left.\frac{d^2\mathcal{R}}{ds^2}\right\vert_{s^*=\ \frac{(2q+1)\pi}{2}}=\frac{-3(a^2 - b^2)}{b}<0$ as $a>b>0$. Thus from  second order sufficient conditions of minimization,  $s^*=q\pi$ minimizes $\mathcal{R}$ and the minimum radius of curvature for the ellipse is $\mathcal{R}_{min}=\left.\mathcal{R}\right\vert_{s^*=q\pi} = \frac{b^2}{a}$

%Solving to extremize $\mathcal{R}$ with resect to $s$, yields $s^*=q\pi,\ (2q+1)\frac{\pi}{2}$ for some $q\in \Bbb{Z}^+$
%Thus to make sure that the algorithm selects an ellipse that
\end{proof}

\bibliographystyle{ieeetr}
\bibliography{References}

\begin{thebibliography}{10}

\bibitem{targetsurvey}
C.~Robin and S.~Lacroix, ``Multi-robot target detection and tracking: taxonomy
  and survey,'' {\em Autonomous Robots}, vol.~40, no.~4, pp.~729--760, 2016.

\bibitem{frew_circle_standoff}
E.~Frew and D.~Lawrence, ``Cooperative stand-off tracking of moving targets by
  a team of autonomous aircraft,'' in {\em AIAA Guidance, Navigation, and
  Control Conference and Exhibit}, p.~6363, 2005.

\bibitem{tsourdos_journal}
H.~Oh, S.~Kim, H.-S. Shin, and A.~Tsourdos, ``Coordinated standoff tracking of
  moving target groups using multiple {U}{A}{V}s,'' {\em IEEE Transactions on
  Aerospace and Electronic Systems}, vol.~51, no.~2, pp.~1501--1514, 2015.

\bibitem{frew_racetrack}
E.~W. Frew, D.~A. Lawrence, and S.~Morris, ``Coordinated standoff tracking of
  moving targets using {L}yapunov guidance vector fields,'' {\em Journal of
  Guidance, Control, and Dynamics}, vol.~31, no.~2, pp.~290--306, 2008.

\bibitem{frew_ellipse}
E.~W. Frew, ``Cooperative standoff tracking of uncertain moving targets using
  active robot networks,'' in {\em Robotics and Automation, 2007 IEEE
  International Conference on}, pp.~3277--3282, IEEE, 2007.

\bibitem{vfield_cylinder}
D.~Lawrence, ``Lyapunov vector fields for {U}{A}{V} flock coordination.,'' in
  {\em 2nd AIAA Unmanned Unlimited Conference (Workshop, and Exhibit),},
  pp.~1--8, 2003.

\bibitem{beard_journal}
D.~R. Nelson, D.~B. Barber, T.~W. McLain, and R.~W. Beard, ``Vector field path
  following for miniature air vehicles,'' {\em IEEE Transactions on Robotics},
  vol.~23, no.~3, pp.~519--529, 2007.

\bibitem{galloway_cyc_pursuit}
K.~S. Galloway and B.~Dey, ``Station keeping through beacon-referenced cyclic
  pursuit,'' in {\em American Control Conference (ACC), 2015}, pp.~4765--4770,
  IEEE, 2015.

\bibitem{ma_moving_cyc_pursuit}
L.~Ma and N.~Hovakimyan, ``Cooperative target tracking in balanced circular
  formation: Multiple {U}{A}{V}s tracking a ground vehicle,'' in {\em American
  Control Conference (ACC), 2013}, pp.~5386--5391, IEEE, 2013.

\bibitem{ma_moving_var_rad}
L.~Ma and N.~Hovakimyan, ``Cooperative target tracking with time-varying
  formation radius,'' in {\em European Control Conference (ECC), 2015},
  pp.~1699--1704, IEEE, 2015.

\bibitem{zhang_formation}
M.~Zhang and H.~H. Liu, ``Cooperative tracking of a moving target using
  multiple fixed-wing {U}{A}{V}s,'' {\em Journal of Intelligent \& Robotic
  Systems}, vol.~81, no.~3-4, pp.~505--529, 2016.

\bibitem{leonard_formation}
D.~Paley, N.~E. Leonard, and R.~Sepulchre, ``Collective motion: Bistability and
  trajectory tracking,'' in {\em 43rd IEEE Conference on Decision and Control
  (CDC)}, vol.~2, pp.~1932--1937, IEEE, 2004.

\bibitem{oliveira_lemniscate}
T.~Oliveira, A.~P. Aguiar, and P.~Encarna{\c{c}}{\~a}o, ``A convoy protection
  strategy using the moving path following method,'' in {\em Unmanned Aircraft
  Systems (ICUAS), 2016 International Conference on}, pp.~521--530, IEEE, 2016.

\bibitem{spry_lemniscate}
S.~C. Spry, A.~R. Girard, and J.~K. Hedrick, ``Convoy protection using multiple
  unmanned aerial vehicles: organization and coordination,'' in {\em American
  Control Conference, 2005. Proceedings of the 2005}, pp.~3524--3529, IEEE,
  2005.

\bibitem{magnus_dubins}
X.~C. Ding, A.~R. Rahmani, and M.~Egerstedt, ``Multi-{U}{A}{V} convoy
  protection: An optimal approach to path planning and coordination,'' {\em
  IEEE Transactions on Robotics}, vol.~26, no.~2, pp.~256--268, 2010.

\bibitem{khalil}
H.~K. Khalil, {\em Noninear Systems}.
\newblock Prentice-Hall, New Jersey, second~ed., 1996.

\bibitem{Alekseev}
V.~M. Alekseev, ``An estimate for the perturbations of the solutions of
  ordinary differential equations (in {R}ussian),'' {\em Vestnik Moskov. Uni.
  Series I Mat. Mech.}, vol.~2, pp.~28--36, 1961.

\bibitem{Brauer}
F.~Brauer, ``Perturbations of nonlinear systems of differential equations,''
  {\em Journal of Mathematical Analysis and Applications}, vol.~14, no.~2,
  pp.~198--206, 1966.

\bibitem{Hirsch}
M.~W. Hirsch, ``Convergent activation dynamics in continuous time networks,''
  {\em Neural networks}, vol.~2, no.~5, pp.~331--349, 1989.

\bibitem{DaCunha}
J.~J. DaCunha, ``Stability for time varying linear dynamic systems on time
  scales,'' {\em Journal of Computational and Applied Mathematics}, vol.~176,
  no.~2, pp.~381--410, 2005.

\bibitem{Solo}
V.~Solo, ``On the stability of slowly time-varying linear systems,'' {\em
  Mathematics of Control, Signals, and Systems (MCSS)}, vol.~7, no.~4,
  pp.~331--350, 1994.

\end{thebibliography}
\end{document}